\documentclass[twocolumn]{article}
\usepackage[left=1.5cm,right=1.5cm,top=1.5cm,bottom=1cm,includeheadfoot]{geometry}
\usepackage{authblk}

\usepackage{microtype}
\usepackage{graphicx}
\usepackage{subfigure}
\usepackage{booktabs} % for professional tables
\usepackage{adjustbox}
\usepackage{multirow}
\usepackage{hyperref}

\usepackage{natbib}

\usepackage{amsmath}
\usepackage{amssymb}
\usepackage{mathtools}
\usepackage{amsthm}
\usepackage{nccmath}

\usepackage[capitalize,noabbrev]{cleveref}

\theoremstyle{plain}
\newtheorem{theorem}{Theorem}[section]

\newtheorem{lemma}[theorem]{Lemma}
\newtheorem{corollary}[theorem]{Corollary}
\theoremstyle{definition}
\newtheorem{definition}[theorem]{Definition}
\newtheorem{assumption}[theorem]{Assumption}
\theoremstyle{remark}

\usepackage[textsize=tiny]{todonotes}

\title{Emergence of heavy tails in homogenized stochastic gradient descent}

\author[1]{Zhe Jiao}
\author[2]{Martin Keller-Ressel}
\affil[1]{School of Mathematics and Statistics, Northwestern Polytechnical University, Xi'an, China}
\affil[2]{Department of Mathematics, TU Dresden, Germany}
\affil[2]{ScaDS.ai Center for scalable data analytics and artificial intelligence, Leipzig/Dresden, Germany}
\begin{document}
\maketitle

\begin{abstract}
It has repeatedly been observed that loss minimization by stochastic gradient descent leads to heavy-tailed distributions of neural network parameters.  Here, we analyze a continuous diffusion approximation of SGD, called homogenized stochastic gradient descent, show that it behaves asymptotically heavy-tailed, and give explicit upper and lower bounds on its tail-index. We validate these bounds in numerical experiments and show that they are typically close approximations to the empirical tail-index of SGD iterates. In addition, their explicit form enables us to quantify the interplay between optimization parameters and the tail-index. Doing so, we contribute to the ongoing discussion on links between heavy tails and the generalization performance of neural networks as well as the ability of SGD to avoid suboptimal local minima.
\end{abstract}

\section{Introduction}\label{intro}
Stochastic gradient descent (SGD) is the cornerstone of optimization in modern deep learning (cf. \citet{Bottou2018}). In contrast to deterministic methods, it introduces stochasticity to the optimization procedure and therefore has to be analyzed from a probabilistic viewpoint. For instance, it has been observed by \cite{mm2019, ssg2019, hm2021, GSZ2021} and others, that the distributions of neural network parameters under loss minimization by SGD are typically \textit{heavy-tailed}. This heavy-tailed behavior has been linked to the generalization performance of neural networks: \citet{ssg2019} give evidence that the extreme realizations of heavy-tailed random variables allow SGD to escape local minima of the loss landscape, and \citet{hm2021} argue for a negative correlation between the parameter distributions's tail-index\footnote{The tail-index is a quantitative measure of heavy-tailedness, with a smaller tail-index indicating increased heaviness of tails; see Section~\ref{sec:ht}.} and the network's generalization performance. For these reasons, it is important to understand the origin and effects of heavy-tailed behavior of neural network parameters in SGD. An important step in this direction has been taken in \cite{GSZ2021}, where the tail behavior of SGD iterates is characterized in dependence on optimization parameters, dimension and Hessian curvature at the loss minimum. One limitation of \cite{GSZ2021} is that this link is described only qualitatively, but not quantitatively. Here, we provide an alternative approach through analyzing homogenized stochastic gradient descent, a diffusion approximation of SGD introduced in \cite{paquette2022implicit, MLLU2022}. Leveraging It\^o calculus for diffusion processes, we are able to provide more precise bounds and estimates of the tail behavior of SGD iterates, which we subsequently validate in numerical experiments.
\subsection{Our contribution}
Our contribution to the analysis of heavy-tailed phenomena in SGD can be summarized as follows:
\begin{itemize}
\item We introduce a new method, namely comparison results in \textit{convex stochastic order} for homogenized stochastic gradient descent. These comparison results, derived in Section~\ref{sec:theory} allow us to link SGD to the well-studied class of \textit{Pearson Diffusions} (cf. \citet{Forman2008}) and then to obtain bounds for their tail-index.
\item Contrary to \cite{GSZ2021}, who describe the tail-index only implicitly (observing phase-transitions between different regimes) our tail-index bounds are fully explicit. Moreover, their explicit form is validated in numerical experiments in Section~\ref{sec:experiments}.
\item Our results suggest (skew) t-distributions as surrogate for parameter distributions in neural networks under SGD, in contrast to the earlier work of \citep{GSZ2021} where $\alpha$-stable distributions have been suggested.
\item Finally, our results strongly challenge the claim that the \textit{`observed heavy-tailed behavior of SGD in practice cannot be accurately represented by an SDE driven by a Brownian motion'} put forward in \cite{ssde2020}. Our modeling approach is based on hSGD -- an SDE driven by Brownian motion -- which asymptotically exhibits heavy-tailed behavior with a tail-index that, in experiments, closely matches the empirical tail-index of SGD iterates on real data. 
\end{itemize}

\section{Background}
\label{sec:background}
\subsection{Empirical risk minimization}
The general framework for training deep neural networks is to solve the problem of empirical risk minimization (ERM)
\begin{equation} \tag{ERM}
	\min_{x\in\mathbb{R}^{d}}\left\{f(x):=\frac{1}{n}\sum_{i=1}^{n}f_{i}(x)\right\} \label{erm}
\end{equation}
where $f_i$ denotes the loss induced by the data point $a_i \in \mathbb{R}^d$ with label/response $b_i \in \mathbb{R}$, and $f$ is the \textit{empirical risk} over the training data. For our theoretical and numerical analysis of heavy-tailed phenomena, as in \cite{GSZ2021}, we assume a quadratic structure of $f_i(x)$ with the understanding that a smooth loss landscape can typically be well-approximated by a quadratic function around a local minimum. Thus, we specify the function $f_i$ by setting
\[
	f_{i}(x)= \frac{1}{2}(a_i \cdot x - b_i)^2 + \frac{\delta}{2}\|x\|^2:=  L_i(x) + \frac{\delta}{2}\|x\|^2,
\]
where $L_i(x) $ is the unregularized loss on the $i$-th data point and $\delta \ge 0$ a regularization parameter. This is the same loss function that is used for \emph{ridge regression} (cf. \citet{tibshirani}). We arrange the training data into a design matrix $ A\in \mathbb{R}^{n\times d}$ and label vector $ b\in \mathbb{R}^{n}$, whose $i$-th row are given by $a_i$ and $b_i$ respectively. Thus, we have
\begin{equation*}
	f(x) =\frac{1}{2n}\| Ax- b\|^{2} + \frac{\delta}{2}\|x\|^2 := \frac{1}{n}L(x) + \frac{\delta}{2}\|x\|^2
\end{equation*}
with gradient given by $\nabla f(x) = \frac{1}{n}\nabla L(x) + \delta x$.

\subsection{Stochastic gradient descent}

The standard approach to solve the problem (\ref{erm}) is to use stochastic gradient descent (SGD) or any of its generalizations involving momentum, adaptive learning rates, gradient rescaling, etc. (cf. \citet{Goodfellow2016, Bottou2018}).
As a first step, we consider plain SGD with constant learning rate $\gamma$, which can be written in recursive form as
\begin{equation} \tag{SGD}
	x_{k+1} = x_{k} - \gamma \nabla f_{\Omega_k}(x_k) \label{sgd}
\end{equation}
where $\nabla f_{\Omega_k}(x_k)=\frac{1}{B}\sum_{i\in\Omega_k}f_i(x)$ and $\Omega_k$ is a batch of size $B\geqslant 1$ sampled uniformly and independently from $\{1, \cdots, n\}$. It will be convenient to rewrite (\ref{sgd}) as
\begin{equation}
	x_{k+1} = x_{k} - \gamma \nabla f(x_k) + \gamma\varepsilon(x_k)  \label{sgd_2}
\end{equation}
where the gradient noise is given by
\begin{equation}
	\varepsilon(x_k) = -[\nabla f_{\Omega_k}(x_k) - \nabla f(x_k)]. \label{noise}
\end{equation}
Note that the gradient noise is unbiased (i.e. $\mathbb{E}\varepsilon(x)= 0$) with covariance matrix given by\footnote{Full derivation given in Supplement \ref{sm:converance}.}
\begin{align*}\label{covariance}
&C(x)  := \mathbb{E}\left[\varepsilon(x)^\top \varepsilon(x)\right] \\
&= \frac{1}{B} \left(\frac{1}{n}\sum_{i=1}^{n} \nabla L_i(x)^\top \nabla L_i(x) - \frac{1}{n^2}\nabla L(x)^\top \nabla L(x)\right). \notag
\end{align*}

\subsection{Homogenized Stochastic Gradient Descent}
Homogenized stochastic gradient descent (hSGD), introduced concurrently in \cite{Paquette+22} and \cite{MLLU2022}, is a diffusion approximation of SGD described by a stochastic differential equation (SDE) driven by Brownian motion. 
It is obtained by matching the drift and diffusion coefficient of the SDE to the expectation and to the covariance of the gradient noise (\ref{noise}) and by applying the approximation (cf. \citet{Paquette+22})
\begin{align*}
	C(x) &\overset{\textbf{I}}{\approx} \frac{1}{B}\left(\frac{1}{n}\sum_{i=1}^{n} \nabla L_{i}(x)^{\top}  \nabla L_{i}(x)\right) \\ &=  \frac{1}{B}\left[\frac{1}{n}\sum_{i=1}^{n}(a_i \cdot x - b_i)^2 a_i^{T}a_i\right]\\
	&\overset{\textbf{II}}{\approx} \frac{2}{n^2B}L(x)\nabla^2L(x),
\end{align*}
in which 
\begin{itemize}
\item the approximation \textbf{I} is true due to the fact that the gradient noise variance dominates the gradient mean near minima which is based on \cite{Smith2018};
\item the approximation \textbf{II} comes from the decoupling approximation (cf. \citet{MLLU2022}).
\end{itemize}
Note that hSGD differs from the well-known Ornstein-Uhlenbeck approximation of \cite{Mandt2016, Jastrzebski2017}, which uses a deterministic approximation of the diffusion coefficient, whereas the diffusion coefficient of hSGD is stochastic. \citet{Paquette+22} show both analytically and in experiments that hSGD approximates the dynamics of SGD with high accuracy, in particular in large dimension. In our notation, hSGD for empirical risk minimization is given by 
\begin{equation} \tag{hSGD}
	dX_t = -\gamma\nabla f(X_t)dt + \gamma\sqrt{\frac{2}{n^2 B}L(X_t)\nabla^2L(X_t)}dW_t, \label{hsgd}
\end{equation}
where $(W_t)_{t \ge 0}$ is $d$-dimensional standard Brownian motion\footnote{We remark that \cite{Paquette+22} assume a batch size of $B=1$; the derivation of \cite{MLLU2022}, however, does not restrict $B$.}.

Following \cite{Paquette+22}, the stochastic differential equation \eqref{hsgd} can be simplified by using the singular value decomposition of the design matrix $A$. In detail, let $A=P\Sigma Q^{\mathrm{T}} $ be the singular value decomposition of $A$, where $Q$ is $d$-by-$d$ and satisfies $Q^{\mathrm{T}}  Q = I$, $P$ is $n$-by-$d$ and satisfies $P^T P= I$ and 
\[
	\Sigma = \textrm{diag}\{\lambda_j\}, \quad \lambda_1 \geqslant \lambda_2 \geqslant \cdots \geqslant \lambda_d  \ge 0.
\]
At this point we impose the following mild assumption:
\begin{assumption}\label{ass1}
All Eigenvalues of $A$ are strictly positive and $b$ is not in the column space of $A$.
\end{assumption}
It is easily verified that under Assumption~\ref{ass1} $x^\ast = (A^{\mathrm{T}}  A)^{-1}A^{\mathrm{T}}  b$ is the global minimum of the unregularized loss function $L$. We set 
\begin{align*}
\alpha &=  (\alpha_i)_{i=1}^d = \left[\left(1-\frac{1}{n}\right) \Sigma^{\mathrm{T}}  \Sigma  - \delta I_d \right] Q^{\mathrm{T}}  x^\ast\\
\beta &=  b^{\mathrm{T}} (PP^{\mathrm{T}} - I) b > 0\\
Y_t &= (Y_t^i)_{i=1}^d = Q^{\mathrm{T}}  X_t - Q^{\mathrm{T}}  x^\ast
\end{align*}
where the strict positivity of $\beta$ follows from Assumption~\ref{ass1}, and obtain the system of SDEs
\begin{multline}\label{sde1}
	dY^{i}_t = -\gamma \left[ \left(\frac{ \lambda^2_i}{n} + \delta\right)Y^{i}_t - \alpha_i \right]\,dt \\ + \gamma \lambda_i \sqrt{\frac{1}{B n^2} \left[\sum_{j=1}^{d}(\lambda_j Y^{j}_t)^2 + \beta\right]} dB^{i}_t 
\end{multline} 
for the `centered principal components' $(Y_t^1, \dotsc, Y_t^d)$ of \eqref{hsgd}, on which our analysis will be based.  Note that the processes $Y_t^i$ are only coupled through the summation term in their diffusion coefficients. 

\subsection{Heavy-Tailed Distributions}\label{sec:ht}
We start by collecting some definitions related to heavy-tailed distributions and their tail-index (cf. \citet{resnick2007}).

\begin{definition}
A distribution function $F(z)$ is said to be \emph{heavy-tailed} (at the right end) if and only if
\[
	\limsup_{t\rightarrow\infty}\frac{1- F(x)}{e^{-sz}} = \infty, \quad \textrm{for all $s>0$}.
\]
A real-valued random variable is said to be heavy-tailed if its distribution function is heavy-tailed.
\end{definition}

\begin{definition} \label{def_2}
An $\mathbb{R}^d-$valued random vector $X$ is heavy-tailed if $u^{\mathrm{T}} X$ is heavy-tailed for some vector $u\in\mathbb{S}^{d-1}:=\{u\in\mathbb{R}^{d}: \|u\| = 1\}$.
\end{definition}

\begin{definition}
The \emph{tail-index} of an $\mathbb{R}^d-$valued random vector $X$ is defined as
\begin{equation*}
\eta := \sup \{p \ge 0: \mathbb{E}[|X|^p] < \infty\} \in [0,\infty]
\end{equation*}
\end{definition} 
In particular, a finite tail-index $\eta < \infty$ implies heavy-tailedness of $X$, and lower values of $\eta$ signify increased heaviness of tails and more extremal behavior. A tail-index of $\eta < 2$, for example, implies infinite variance and $\eta < 1$ implies non-existence of even the mean of $X$. Examples of heavy-tailed distributions are the lognormal distribution, the $t$-distribution, the Pareto (power-law) distribution, and $\alpha$-stable distributions.

Finally, we introduce a definition related to the asymptotic behavior of stochastic processes.
\begin{definition}
Let $X = (X_t)_{t \geq 0}$ be a stochastic process. 
The \emph{asymptotic tail-index} of $X$ is defined as
\begin{equation}\label{eq:tail_index}
\eta := \sup \{p \ge 0: \limsup_{t \to \infty} \mathbb{E}[|X_t|^p] < \infty\}.
\end{equation}
\end{definition}

\subsection{Pearson Diffusions}
We perform a convenient rescaling of \eqref{sde1} by setting, for $i \in \{1, \dotsc, d\}$,
\begin{align} \label{substitute}
	Z^{i}_t  &= \frac{\lambda_i}{\sqrt{\beta}}Y^{i}_t, \quad \theta_i = \gamma \left( \frac{\lambda^2_i}{n} + \delta \right) > 0, \\
	\mu_i &=  \frac{n\lambda_i\alpha_i}{\sqrt{\beta}(\lambda^2_i +n\delta)}, \quad a_i = \frac{\gamma\lambda_i^4}{2nB(\lambda_i^2 + n\delta)}, \notag
\end{align}
which recasts the system \eqref{sde1} to
\begin{equation}
	dZ^{i}_t = -\theta_i(Z^{i}_t -\mu_i)dt + \sqrt{2\theta_i a_i (|Z_t|^2 + 1)}dB^{i}_t.  \label{eq:sde_Z}
\end{equation}
These SDEs now have a clear structural resemblance to the system of independent one-dimensional SDEs
\begin{equation}
	d\hat Z^i_t = -\theta(\hat Z^i_t -\mu_i)dt + \sqrt{2\theta_i a_i ((\hat Z^i_t)^2 + 1)}dB^i_t,  \label{PearsonIV}
\end{equation}
with the only difference given by the coupling of \eqref{eq:sde_Z} through the $|Z_t|^2$-term in the diffusion coefficient. The components of \eqref{PearsonIV} are independent \textit{Pearson diffusions}. Pearson diffusions are a flexible class of SDEs with a unified theory for statistical inference and with stationary distributions known as Pearson distributions (cf. \citet{Forman2008}). The stationary distribution of $\hat Z^i_t$ described by \eqref{PearsonIV} is called Pearson's type IV distribution (or skew $t$-distribution) and has the un-normalized density
\begin{align}
	p_i(u) \propto & \left[1+ \left(\tfrac{u}{\sqrt{\nu_i}} + \mu_i\right)^2\right]^{-\frac{\nu_i + 1}{2}} \cdot \notag \\ &\exp\left\{\mu_i(\nu_i -1)\arctan\left(\tfrac{u}{\sqrt{\nu_i}} + \mu_i\right)\right\}\label{eq:stationary_law}
\end{align}
with $\nu_i = a_i^{-1} + 1$. If $\mu_i = 0$, Pearson's type IV distribution will be a scaled $t$-distribution. Figure \ref{fig_1} (g) demonstrates the change in the complementary cumulative distribution functions of the $t$-distribution as $\nu_i$ increases. It is easily seen that the Pearson type IV distribution is heavy-tailed with tail-index given by $\nu_i$, thus providing a first connection between the SDE-approach and the emergence of heavy-tails. We also emphasize the contrast to \cite{GSZ2021}, where the different class of $\alpha$-stable distributions is used to describe the asymptotic behavior of SGD iterates.

%\begin{figure}[htbp]
%\begin{center}
%\centerline{\includegraphics[width=15em]{fig/t_distribution_cdf.png}}
%\caption{Empirical cumulative distributions function by using random numbers generated by $t$-distribution with different tail indices.}
%\label{fig_t}
%\end{center}
%\end{figure}

\section{Theoretical results}
\label{sec:theory}
\subsection{Comparison to Pearson Diffusion}

\begin{theorem} \label{thm:main1}
For $i = 1, \cdots, d$, let $(Z^{i}_t)_{t\geqslant 0}$ be the components of the rescaled \eqref{hsgd} from \eqref{eq:sde_Z} and $(\hat{Z}^{i}_t)_{t\geqslant 0}$ be the independent Pearon diffusion from \eqref{PearsonIV}. Then for any $t\geqslant0$ and convex function $g: \mathbb{R}\rightarrow\mathbb{R}$ it holds that
\begin{equation}\label{eq:conv_order}
\mathbb{E}[g(Z^i_t)] \ge \mathbb{E}[g(\hat Z^i_t)].
\end{equation}
In particular this implies the ordering of $p$-moments
\begin{equation}\label{eq:p_order}
\mathbb{E}[|Z^i_t|^p] \ge \mathbb{E}[|\hat Z^i_t|^p]
\end{equation}
for all $p \ge 1$. 
\end{theorem}
The ordering of $Z^i_t$ and $\hat Z^i_t$ given by \eqref{eq:conv_order} is also known as \emph{convex stochastic order}; see \cite{shaked2007}. Note that finiteness of the expectations does not need to be assumed, i.e., the inequalities also hold if one of the expectations takes the value $+\infty$. Comparison results in stochastic order for SDEs have been shown for example in \cite{bergenthum2007}. However, since none of the results can be applied directly in our setting, we will give a self-contained proof.

Comparison results for SDEs generally require two conditions (cf. \citet{bergenthum2007}): An ordering between the drift- and diffusion-coefficients of the two SDEs, and the `propagation-of-order'-property for one of the processes. In our case, the SDEs \eqref{eq:sde_Z} and \eqref{PearsonIV} can be represented -- component by component -- in the form
\begin{align*}
%\begin{aligned}
	dZ^i_t &= b_i(Z_t)dt + \sigma_i(Z_t) dB^i_t,\\
	d\hat{Z}^i_t& = b_i(\hat{Z}_t)dt + \hat{\sigma}_i(\hat{Z}^i_t)dB^i_t,
%\end{aligned}
\end{align*}
where 
\begin{align*}
b_i(z) = -\theta_i (z_i - \mu_i),\\
\sigma_i^2(z) = 2 \theta_i a_i (|z|^2 + 1) \quad \text{and} \quad 
\hat \sigma_i(z_i)^2 &=  2 \theta_i a_i (z_i^2 + 1).
\end{align*}
While the drift coefficients are identical, the diffusion coefficients satisfy the inequality $\sigma_i(z) \ge \hat \sigma_i(z)$ for all $z \in \mathbb{R}^{d}$ and $i =1, \dotsc, d$. Also note that all coefficients are Lipschitz continuous and of bounded growth, such that the standard assumptions for uniqueness and existence of strong SDE solutions are satisfied.
Moreover, the SDEs for $\hat Z_t^i$ are decoupled and each is a Markov diffusion with generator given by
\begin{eqnarray*}
\begin{aligned}
	\hat{\mathcal{L}}_i = b_i(x) \partial_x + \frac{\sigma_i(x)^2}{2} \partial_{x x},
\end{aligned}
\end{eqnarray*}
where $x$ denotes the scalar state variable of $\hat Z^i$. Let $C_{P}^{l}(\mathbb{R})$ denote the subspace of $C^{l}$-functions for which all derivatives up to order $l$ have polynomial growth. Suppose that $g\in C_{P}^{l}(\mathbb{R})$. From Theorem 4.8.6 in \cite{KP} the backward functional 
\[
	\mathcal{G}_i(t, x) = \mathbb{E}[g(\hat{Z}^i_T)|\hat{Z}^i_t = x], \quad t\in[0, T],
\]
satisfies the backward Kolmogorov equation
\begin{align}\label{eq:bw_eq}
	\partial_t \mathcal{G}_i(t, x) + \hat{\mathcal{L}}_i\mathcal{G}_i(t, x) &= 0 \quad t < T, \\
	\mathcal{G}_i(T, x) &= g(x). \notag
\end{align}
with $\partial_t \mathcal{G}_i$ continuous and $\mathcal{G}_i(t, \cdot)\in C_{P}^{l}(\mathbb{R})$ for each $t\in[0, T]$.\\
The following Lemma on the \emph{propagation-of-order} property of $\hat Z$ can be shown by Euler-Maruyama approximation; see Supplement \ref{sm:lemma_proof} for the complete proof.
\begin{lemma}\label{lem:POg}
If $g\in C_{P}^{l}(\mathbb{R})$ is convex, so is $\mathcal{G}_i(t, \cdot)$ for all $t\in[0, T]$ and $i =1, \dotsc, d$.
\end{lemma}
Finally, we need a technical result that shows that each process $\mathcal{G}_i(z,\hat Z^i_t)_{t \in [0,T]}$ is of `class $\mathrm{(D)}$' a proof is given in the Supplement \ref{sm:lemma_proof}.\footnote{A stochastic process $(X_t)_{t \in I}$ is of class (D), if the set $\{X_{\tau}: \text{$\tau$ is $I$-valued stopping time}\}$ is uniformly integrable (cf. Definition 4.8 in \cite{karatzas2012}).}
\begin{lemma}\label{lem:class_DL}
For each $i=1, \dotsc, d$, the process  $\mathcal{G}_i(t, Z_t^i)_{t \in [0,T]}$ is of class $\mathrm{(D)}$.
\end{lemma}

We are now prepared to give the proof of our first main result.
\begin{proof}[Proof of Theorem~\ref{thm:main1}]
Let $g$ be a convex function and assume for now that $g \in C_P^2(\mathbb{R})$. Define the local martingale 
\[L_t = \int_{0}^{t}\partial_{x}\mathcal{G}_i(s, Z^i_s)\sigma_i(Z_s)dB^i_s\]
Using It\^o's formula in the first step and \eqref{eq:bw_eq} in the second step, we have

\begin{align} \label{estimate}
%\medmath{
\mathcal{G}_i&(t, Z^i_t)   - \mathcal{G}_i(0, Z^i_0)\\ 
=  \notag & \medint\int_{0}^{t}\partial_{t}\mathcal{G}_i(s, Z^i_s)ds + \phantom{x}\\
\notag & \medint \int_{0}^{t}\left(b_i(Z_t^i) \partial_x + \tfrac{\sigma_i^2(Z_t)}{2} \partial_{xx}\right)\mathcal{G}_i(s, Z^i_s)ds +  L_t \\
= \notag & - \medint \int_{0}^{t}\hat{\mathcal{L}}_i\mathcal{G}_i (s, Z_s)ds +  \phantom{x}\\
\notag &\medint \int_{0}^{t}\left(b_i(Z_t^i) \partial_x + \tfrac{\sigma_i^2(Z_t)}{2} \partial_{xx}\right)\mathcal{G}_i(s, Z^i_s)ds +  L_t  \\
= \notag &\tfrac{1}{2}\medint \int_{0}^{t} [\sigma_i^2(Z_s) - \hat{\sigma}_i(Z^i_s)]\partial_{xx} \mathcal{G}_i(s, Z^i_s)ds + L_t.
\end{align}

By $\mathcal{G}_i(t, \cdot)\in C^2_P(\mathbb{R})$ and Lemma \ref{lem:POg} we obtain $\partial_{xx} \mathcal{G}_i(s, \cdot) \geqslant 0$ for all $i \in \{1, \dotsc, d\}$. Thus, due to the ordering of $\sigma_i^2$ and $\hat{\sigma}_i^2$, the first term in the right hand side of (\ref{estimate}) is nonnegative. Since $L$ is a continuous local martingale with zero initial data, it follows that $\mathcal{G}_i(t, Z_t)- \mathcal{G}_i(0, Z_0) $ is a  local submartingale.

Let $\tau_n$ be a localizing sequence for $\mathcal{G}_i(t, Z_t)$. For all $t\in[0, T]$, we have
\begin{equation} \label{estimate2}
	\mathcal{G}_i(t\land\tau_n, Z_{t\land\tau_n}) - \mathcal{G}_i(0, Z_0) \xrightarrow[n\rightarrow\infty]{a.s.} \mathcal{G}_i(t, Z_t) - \mathcal{G}_i(0, Z_0).
\end{equation}
Since $\mathcal{G}_i(t, Z_t)$ is a process of class $\mathrm{(D)}$ or locally $L^p$-bounded, $p>1$, it follows that $\mathcal{G}_i(t\land\tau_n, Z_{t\land\tau_n}) - \mathcal{G}_i(0, Z_0)$ is uniformly integrable. Combining almost-sure convergence with the uniformly integrable property, it implies that the convergence (\ref{estimate2}) also takes place in $L^1$,
and therefore, $\mathcal{G}_i(t, Z_t) - \mathcal{G}_i(0, Z_0) $ is a submartingale.
By taking expectations on both sides of (\ref{estimate}) and using the fact that $Z_0 = \hat{Z}_0 $, we obtain the comparison result
\begin{equation}\label{eq:conv_smooth_order}
	\mathbb{E}g(Z^i_T) = \mathbb{E}\mathcal{G}_i(T, Z^i_T) \geqslant \mathcal{G}(0, Z^i_0) = \mathbb{E}[g(\hat{Z}^i_T)]
\end{equation}
for all convex $g \in C_P^2(\mathbb{R})$. 

Now let $g$ be arbitrary convex function on $\mathbb{R}$. From Theorem~3.1.4 in \cite{hiriart1996} we can find, for each $n \in \mathbb{N}$ a convex Lipschitz function $\tilde g_n$ such that $\tilde g_n \le g$ in $[-n,n]$ and $\tilde g_n \le g$ in $\mathbb{R} \setminus [-n,n]$. By \cite{azagra2013} we can find further smooth convex functions $g_n \in C_\text{Lip}^\infty(\mathbb{R})$ such that $\tilde g_n - \tfrac{1}{n} \le g_n \le \tilde g_n$ on all of $\mathbb{R}$. It follows that the sequence $g_n$ converges pointwise to $g$ from below. We observe that $C_\text{Lip}^\infty(\mathbb{R}) \subset C_P^2(\mathbb{R})$ and equation \eqref{eq:conv_order} now follows from \eqref{eq:conv_smooth_order} by monotone convergence. Finally, equation \eqref{eq:p_order} follows by choosing the convex function $g(z_i) = |z_i|^p$.
\end{proof}

\subsection{Upper bound for the asymptotic tail-index}
From $X_t = QY_t + x_*$, the triangle inequality and the unitary invariance of the Euclidean norm, it follows that $|Y_t| \le |X_t| + |x_*|$. Thus, we have
\begin{align}\label{eq:triangle}
\frac{\beta^{p/2}}{\lambda_1^p} \mathbb{E}[|Z_t^1|^p] &= \mathbb{E}[| Y_t^1|^p] \le \mathbb{E}[| Y_t|^p]  \notag \\
&\le 2^p\left(\mathbb{E}[|X_t|^p] + |x_*|^p\right).
\end{align}
Now, let $p > \nu_1$. By Theorem~\ref{thm:main1}, Fatou's Lemma, and the properties of the skew $t$-distribution \eqref{eq:stationary_law}
\begin{align}\label{eq:fatou}
\limsup_{t \to \infty} \mathbb{E}[|Z_t^1|^p] &\ge \liminf_{t \to \infty} \mathbb{E}[|Z_t^1|^p] \\
&\ge \liminf_{t \to \infty} \mathbb{E}[|\hat Z_t^1|^p] \ge [|\hat Z_\infty^1|^p] = \infty. \notag
\end{align}
Together with \eqref{eq:triangle} this implies that also 
\[\limsup_{t \to \infty} \mathbb{E}[|X_t|^p] = \infty,\]
and it follows from \eqref{eq:tail_index} that the asymptotic tail-index satisfies $\eta \le p$ for all $p > \nu_1$. Finally, the parameter $\nu_1$ in the limit distribution of $\hat Z^1$ is given by $\nu_1 = 1 + a_1^{-1}$, where $a_1$ can be found in \eqref{substitute}. Thus, we immediately obtain the following result.

\begin{theorem} \label{thm:main2}
The asymptotic tail-index $\eta$ of \eqref{hsgd} has the upper bound
\begin{equation} \label{upper_bound}
	\eta \le \eta^* := 1+\frac{2nB(\lambda_1^2 + n\delta)}{\gamma\lambda_1^4}.
\end{equation}
\end{theorem}

\subsection{Lower bound for the asymptotic tail-index}
For better readability, we rewrite \eqref{hsgd} as the following form
\begin{equation} \label{eq:sde2}
	dX_t = F(X_t)dt + G(X_t) dB_t
\end{equation}
with
\begin{align*}
	F(X_t) &= -\gamma\left[\frac{1}{n}A^{\mathrm{T}}(AX_t -b) + \delta X_t\right], \\
	G(X_t) &= \gamma\sqrt{\frac{1}{n^2 B}\|AX_t - b\|^2A^{\mathrm{T}}A}.
\end{align*}
Under a certain assumption on the learning rate, we can prove (see Supplement \ref{subsec:lower_bound} for details) that for all $\rho\in (0, \eta_*)$
\begin{small}
\begin{align}\label{ass_5_1}
&\limsup_{|x |\rightarrow\infty}\frac{(1 + |x|^2)\left[2x^{\mathrm{T}}F(x)+ |G(x)|^2\right]- (2-\rho)|x^{\mathrm{T}}G(x)|^2}{|x|^4} \notag \\
&\phantom{XX}<-C_1, 
\end{align}
\end{small}
where $C_1$ is a positive constant and
\[
	\eta_* := 1 + \frac{2n(\lambda_1^2 + n\delta)}{\gamma \lambda_1^4} - \frac{ \sum_{i=2}^{d}\lambda_i^2 }{\lambda_1^2} > 0.
\]
By Theorem 5.2 in \cite{LMY2019}, the solution $X_t$ of the SDE \eqref{eq:sde2}) satisfies
\begin{eqnarray*}
	\sup_{0\leqslant t < \infty}\mathbb{E}|X_t|^\rho \leqslant C_2
\end{eqnarray*} 
with $C_2$ a positive constant. Then we have the following theorem. 
\begin{theorem} \label{thm:main3}
Suppose that the learning rate $\gamma$ satisfies
\[
	\gamma < \overline{\gamma} =: \frac{2nB(\lambda_1^2 + n\delta)}{ \lambda_1^2 \sum_{i=1}^{d}\lambda_i^2} ,
\]
then the asymptotic tail-index $\eta$ of \eqref{hsgd} has the lower bound
\begin{equation} \label{lower_bound}
	1+\frac{2nB(\lambda_1^2 + n\delta)}{\gamma\lambda_1^4} - \frac{ \sum_{i=2}^{d}\lambda_i^2}{\lambda_1^2} = \eta_* \le \eta.
\end{equation}
\end{theorem}

\subsection{Discussion of theoretical results}
\label{subsec:discussion}
In comparison to \cite{GSZ2021}, we note the following differences and similarities. In our setting, the data distribution is completely arbitrary, since all results are given conditional on the data matrix $A$. In \cite{GSZ2021} on the other hand, the more restrictive assumption of an isotropic Gaussian data distribution is made. Moreover, our tail-index bounds \eqref{upper_bound} and \eqref{lower_bound} are quantitative and explicit, whereas \citet{GSZ2021} describe when a phase transition of the asymptotic tail-index $\eta$ from $\eta < 2$ to $\eta > 2$ occurs, but do not give further quantitative estimates of $\eta$. 

Some further interesting observations can be made when we consider the dependency of $\eta$ on the meta-parameters of the stochastic gradient descent procedure:
\begin{corollary}
The upper and lower bounds of the tail-index are increasing in the regularization parameter $\delta$ and batch size $B$, and are decreasing in the learning rate $\gamma$ and the first singular value $\lambda_1$ of the data matrix $A$.
\end{corollary}
This result agrees with Theorem~4 in \cite{GSZ2021}, obtained under the assumption of an isotropic data distribution $a_i\sim N(0, \sigma^2 I_d)$, in all aspects, except the dependency on dimension $d$.\footnote{With the key difference that \eqref{upper_bound} and \eqref{lower_bound} give a \emph{quantitative} description of all dependencies, while \cite{GSZ2021} is only \emph{qualitative} in nature.} While \citet{GSZ2021} report decreasing dependency on $d$, our tail-index bounds do not explicitly depend on dimension $d$. Nevertheless, the two results can be reconciled as follows: Under the assumptions in \cite{GSZ2021}, the data matrix $A = (a_i)$ is random with $\mathbb{E}(A^{\mathrm{T}}A)=\sigma^2 I_d$, and the product matrix $W:= A^{\mathrm{T}}A$ follows the so-called Wishart ensemble (cf. \citet{Wishart1928}). 
Moreover, from Theorem 1.1 in \cite{Johnstone2001} it follows that for large $d$ the maximum eigenvalue of $W$ is 
\begin{equation} \label{largesteigen}
	\lambda^2_1 = \sigma^2\left[( \frac{1}{\sqrt{r}}+1)^2 d + r^{\frac{1}{6}}(\frac{1}{\sqrt{r}}+1)^{\frac{4}{3}}d^{\frac{1}{3}}\chi\right],
\end{equation}
where the ratio $r=\frac{d}{n-1}<1$ and the distribution function of the random variable $\chi$ is the well-known Tracy-Widom distribution of order $1$ (cf. \citet{TW1996}). From (\ref{largesteigen}), we can calculate the average of $\lambda^2_1$ as
\[
	\mathbb{E}\left[\lambda^2_1\right] = \sigma^2( \frac{1}{\sqrt{r}}+1)^2 d = \sigma^2 (\sqrt{n-1} + \sqrt{d})^2
\]  
and $\lambda^2_1$ fluctuates around this expectation over a narrow region of width $O(d^\frac{1}{3})$. Substituting $\lambda_1^2$ by its expectation in \eqref{upper_bound} and \eqref{lower_bound} we can now see that $\eta_*$ and $\eta^*$ decrease in both variance $\sigma^2$ and $d$, consistent with \cite{GSZ2021}.

\section{Experiments}
\label{sec:experiments}

Based on the upper and lower bounds in Theorems~\ref{thm:main2}~and~\ref{thm:main3}, we present some experiments to illustrate the tail behavior of SGD and the factors influencing the tail-index.  
The procedure of our experiments contains the following steps.
\begin{enumerate}
\item Given $[\mathrm{data}|b]$, we transform the data to be on a similar scale by the linear scaling
\[
	A = \frac{\mathcal{\mathrm{data}} - \min\{\mathcal{\mathrm{data}}\}}{\max\{\mathcal{\mathrm{data}}\} - \min\{\mathcal{\mathrm{data}}\}}.
\]

\item
Let $K$ be the iteration number of SGD. We apply \ref{sgd} to solve \ref{erm}.
The final state $x_K\in \mathbb{R}^d$ is a random vector. 
\item
Repeat the second step $1000$ times for different initial points and obtain $1000$ different samples of $x_K$.
\item
For further distributional analysis we project $x_K$ via $\mathrm{y}= q_1^{\top}x_K $ on the dominant direction, given by the first right singular vector $q_1$ of $A$. Then we utilize the $1000$ samples to obtain the empirical complementary cumulative distribution function (ccdf) of $\mathrm{y}$.
\end{enumerate}

\subsection{Datasets}

\textit{Synthetic data.}
We first validate our results in the same synthetic setup used in \cite{GSZ2021}. All data points are drawn from isotropic Gaussian distributions, precisely, the $i$-th row of $\mathcal{X}\in\mathbb{R}^{n\times d}$ contains $\chi_i\in\mathbb{R}^{d} \sim \mathcal{N}(0, I_d)$. Then given $x\in\mathbb{R}^{d} \sim \mathcal{N}(0, 3I_d)$ we draw the response vextor $b\in\mathbb{R}^n$ with components $b_i \sim\mathcal{N}(\chi_i x, 3)$. We set the number $n$ of the synthetic data to be $2000$ through our experiments.

\textit{Real data.}
In our second setup we conduct our experiments on the handwritten digits dataset from the Scikit-learn python package (cf. \citet{pedregosa2011}) and a random feature model proposed in \cite{RR2007}. 
The digits dataset contains $n=1797$ images of handwritten digits in a $8\times8$ pixel format. The pixels are stacked into vectors of length $n_0 = 8^2 =64$ resulting in a raw data matrix $\mathcal{Y}\in\mathbb{R}^{n\times n_0}$ and the class label $b_i = \{0, 1, \cdots, 9\}$ is used as response vector. 
For the random feature model, we choose a dimension $d$ and draw a random weight matrix $W\in\mathbb{R}^{n_0\times d}$ having standard Gaussian entries. The feature matrix $W\in\mathbb{R}^{n\times d}$ is given by
\[
	\mathcal{Z} = \sigma\left( \frac{\mathcal{Y}W}{\sqrt{n_0}}\right)\in\mathbb{R}^{n\times d},
\]
where $\sigma(\cdot)$ is a rescaled ReLu activation function.

\textit{Parameters.} Tables \ref{table_1} and \ref{table_2} contain all parameter values used for the figures.

\begin{table}[htbp]
\caption{Parameters used  for Figure \ref{fig_1}}
\label{table_1}
\vskip 0.15in
\begin{adjustbox}{center}
%\begin{center}
\begin{tiny}
%\begin{sc}
%\centering
\scalebox{0.8}{
\begin{tabular}{ccccccccccc}
\toprule
Figure \ref{fig_1} & data & $d$ &  $K$   & $\gamma$ & $\overline{\gamma}$ & $\delta$ & $B$ & $\lambda_1$ & $\eta_*$& $\eta^*$  \\
\toprule
(a), (d), (h)& $\mathcal{X}$  &$200$ & $1000$  & $0.015$ & $0.037$ & $0$ & $1$  & $319.83$  & $3.56$ & $3.61$  \\
(b), (e), (i)&$\mathcal{Y}$  &$64$ &$10000$ & $0.100$ & $0.133$ & $0$ & $1$  & $137.07$ & $2.48$ & $2.91$ \\
(c), (f), (j)&$\mathcal{Z}$  &$200$ &$10000$ & $0.200$ & $0.304$ & $0$ & $1$  & $93.49$ & $2.70$ & $3.06$ \\
\bottomrule

\end{tabular}
}
\end{tiny}
\end{adjustbox}
\end{table}

\begin{table}[htbp]
\caption{Parameters used for Figure \ref{fig_2}}
\label{table_2}
\vskip 0.15in
\begin{adjustbox}{center}
%\begin{center}
\begin{tiny}
%\centering
\scalebox{0.8}{
\begin{tabular}{cccccccc}
\toprule
Figure \ref{fig_2} & data & $d$ &  $K$   & $\gamma$ & $\delta$ & $B$ & $\lambda_1$  \\
\toprule
(a)&$\mathcal{X}$  &$200$ &$1000$ & $0.010$ to $0.025$ & $0$ & $1$  & $353.10$ \\
(b)&$\mathcal{X}$  &$200$ &$1000$ & $0.10$ & $0$ & $1$ to $4$  & $319.83$  \\
(c)&$\mathcal{X}$  & $100$ to $260$ &$1000$ & $0.02$ & $0$ & $1$  & $223.05$ to $360.08$  \\
\midrule
(d)&$\mathcal{Z}$  &$200$ &$10000$ & $0.10$ to $0.25$ & $0$ & $1$  & $93.49$  \\
(e)&$\mathcal{Z}$  &$200$ &$10000$ & $0.10$ & $0$ & $1$ to $4$  & $93.49$  \\
(f)&$\mathcal{Z}$  &$80$ to $360$ &$10000$ & $0.20$ & $0$ & $1$  & $58.25$ to $106.61$  \\
\bottomrule
\end{tabular}
}
\end{tiny}
\end{adjustbox}
\end{table}

\subsection{Empirical results}
\textit{Heavy tailed behavior.} To verify the heavy-tailed behavior of $\mathrm{y}$ as well as our tail-index bounds from Theorems~\ref{thm:main2} and~\ref{thm:main3} and the distributional approximation suggested by \eqref{eq:stationary_law}, we use MLE-estimation to fit our centered data as% From Figure \ref{fig_1} (a), (b) and (c), it can be seen that
\[
	\mathrm{z} := \mathrm{y} - \textrm{mean}\{\mathrm{y}\} \sim \kappa t(\nu).
\]
where $t(\nu)$ denotes a t-distribution with parameter $\nu$ and $\kappa$ is a scaling factor.\footnote{Eq.~\eqref{eq:stationary_law} actually implies a skew t-distribution, but we use a symmetric one to avoid the estimation of an additional parameter $\mu$.} The QQ-plots in Figure~\ref{fig_1}(a), (b) and (c) show that the t-distribution provides an excellent fit to the empirical data, validating our use of Pearson diffusions to approximate SGD. For comparison, we also fit (using MLE-estimation) an $\alpha$-stable distribution, as suggested in \cite{GSZ2021}, to the same data and show the resulting QQ-plots in Figure~\ref{fig_1}(e), (f) and (g). It can be seen that the fitted $\alpha$-stable distribution massively overestimates the heaviness of tails, in particular for the random feature model on real data. We complement these figures by a Kolmogorov-Smirnov test (cf. Chapter 4.4 in \cite{corder2014}) testing for the goodness-of-fit of the t-distribution and the $\alpha$-stable distribution respectively; see~Table \ref{table_3} for results.

\begin{figure}[!htbp]
\centering
\subfigure[]{
\includegraphics[width=0.45\columnwidth]{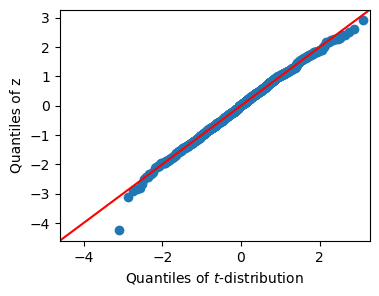}
}
\hspace{0in}
\subfigure[]{
\includegraphics[width=0.45\columnwidth]{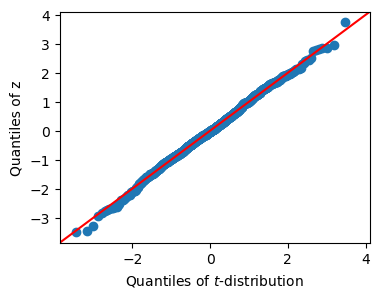}
}
\hspace{0in}
\subfigure[]{
\includegraphics[width=0.45\columnwidth]{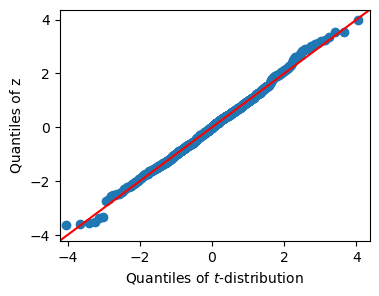}
}
\subfigure[]{
\includegraphics[width=0.45\columnwidth]{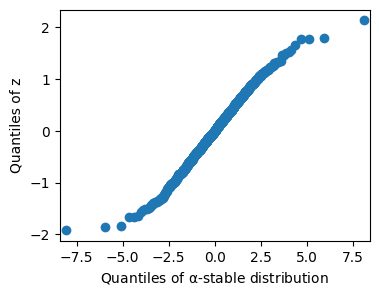}
}
\hspace{0in}
\subfigure[]{
\includegraphics[width=0.45\columnwidth]{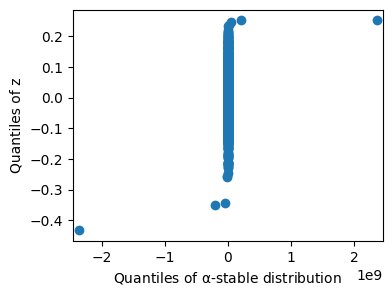}
}
\hspace{0in}
\subfigure[]{
\includegraphics[width=0.45\columnwidth]{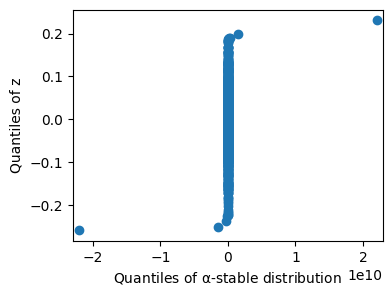}
}
\hspace{0in}
\subfigure[]{
\includegraphics[width=0.45\columnwidth]{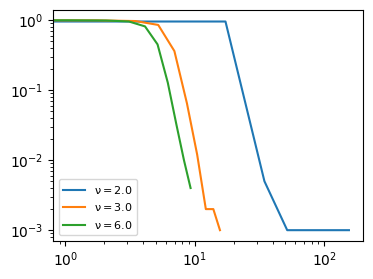}
}
\hspace{0in}
\subfigure[]{
\includegraphics[width=0.45\columnwidth]{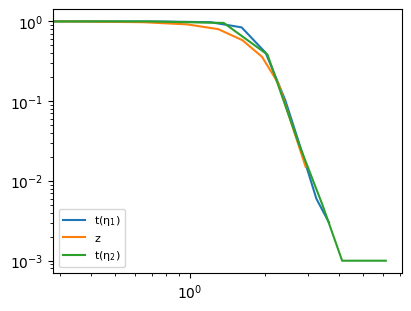}
}
\hspace{0in}
\subfigure[]{
\includegraphics[width=0.45\columnwidth]{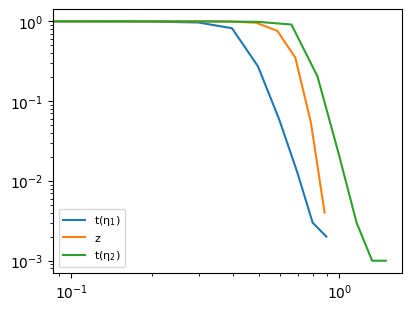}
}
\hspace{0in}
\subfigure[]{
\includegraphics[width=0.45\columnwidth]{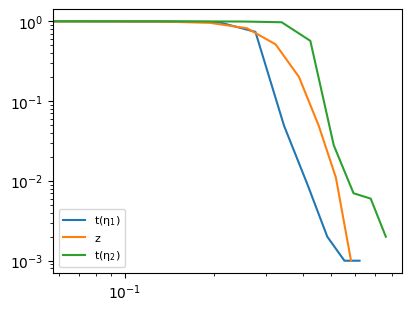}
}
\caption{(a)-(c) Quantile-Quantile plots of fitted t-distribution against empirical SGD iterates; (d)-(f) Quantile-Quantile plots of fitted $\alpha$-stable distribution against empirical SGD iterates. (g) Complementary cumulative distribution function (ccdf) of $t$-distribution with different tail indices; (h)-(j) Comparison between ccdf of empirical data and t-distribution parameterized by upper tail-index bound $\eta^*$ and lower bound $\eta_*$.
\label{fig_1}}
\end{figure}

Moreover, in Figure \ref{fig_1}(h), (i) and (j) we plot (in doubly logarithmic coordinates) the empirical ccdf of the SGD iterates $\mathrm{z}$, together with the ccdf of the t-distribution parametrized by our lower and upper bound $\eta_*$ and $\eta^*$. It can be seen that the empirical ccfd, including its tail, is nicely sandwiched between upper and lower bound, validating Theorems~\ref{thm:main2} and~\ref{thm:main3}. Additionally, we once more confirm the heavy-tailed behavior of SGD iterates as already observed in \cite{ssg2019, hm2021, GSZ2021}.

\begin{table}[htbp]
\caption{Kolmogorov-Smirnov test. The null hypothesis $H_0$ is that two distributions are identical, the alternative $H_1$ is that they are not identical. For the t-distribution we use one-sided null hypothesis $\overline{H}_0$: $F_{\textrm{z}}(x) \geqslant F_{\kappa t(\eta^*)}(x)$ for $x$, the alternative $\overline{H}_1$: $F_{\textrm{z}}(x) < F_{\kappa t(\eta^*)}(x)$ for at least one $x$; The one-sided null hypothesis $\underline{H}_0$: $F_{\textrm{z}}(x) \leqslant F_{\kappa t(\eta_*)}(x)$ for all $x$, the alternative $\underline{H}_1$: $F_{\textrm{z}}(x) > F_{\kappa t(\eta^*)}(x)$ for at least one $x$.}
\label{table_3}
\vskip 0.15in
\begin{adjustbox}{center}
%\begin{center}
\begin{tiny}
%\centering
\scalebox{0.95}{
\begin{tabular}{cccccc}
\toprule
Figure \ref{fig_1} & $\kappa$ & hypothesis & K-S statistic & $p$-value & decesion \\
\toprule
(d)& $1.0$& $H_0$, $H_1$ & $0.6$ & $0.052 > 0.05 $ & not reject $H_0$  \\
\midrule
(e)& $1.0$& $H_0$, $H_1$ & $0.8$ & $0.002 < 0.05$ & reject $H_0$  \\
\midrule
(f)& $1.0$& $H_0$, $H_1$ & $0.9$ & $0.0002 > 0.05$ & reject $H_0$  \\
\midrule
\multirow{2}{0.25cm}{(h)}& \multirow{2}{0.85cm}{$0.320$}& $\overline{H}_0$, $\overline{H}_1$ & $0.2$ & $0.68 > 0.05$ & not reject $\overline{H}_0$  \\
& & $\underline{H}_0$, $\underline{H}_1$ & $0.0$ & $1.00 > 0.05$ & not reject $\underline{H}_0$  \\
\midrule
\multirow{2}{0.25cm}{(i)}&\multirow{2}{0.85cm}{$0.045$}&$\overline{H}_0$, $\overline{H}_1$  & $0.1$ & $0.91 > 0.05$ & not reject $\overline{H}_0$\\
& & $\underline{H}_0$, $\underline{H}_1$& $0.1$ & $0.91 > 0.05$ & not reject $\underline{H}_0$ \\
\midrule
\multirow{2}{0.25cm}{(j)}&\multirow{2}{0.85cm}{$0.050$} &$\overline{H}_0$, $\overline{H}_1$  & $0.0$ & $1.00 > 0.05$ & not reject $\overline{H}_0$\\
& & $\underline{H}_0$, $\underline{H}_1$ & $0.3$ & $0.42 > 0.05$ & not reject $\underline{H}_0$\\
\bottomrule
\end{tabular}
}
\end{tiny}
\end{adjustbox}
\end{table}

\textit{Increasing learning rate / Decreasing batch size.}
To illustrate the effect of the learning rate $\gamma$, we perform a set of experiments with constant batch size $B=1$, varying only the learning rate $\gamma$. Meanwhile, by fixing $\gamma$, we conduct a series of experiments with varying $B$. In Figure \ref{fig_2} (a),(b), (d) and (e) we can see that increasing $\gamma$ and decreasing $B$ leads to decreasing tail-index $\eta$.

\textit{Increasing dimension.} The dimension $d$ affects the upper and lower bounds via the leading singular value $\lambda_1$ of data matrix $A$ constructed by $\mathcal{X}$ (see the discussion in Section \ref{subsec:discussion}), although it does not appear explicitly in $\eta^*$ and $\eta_*$. In Figure \ref{fig_2} (c) and (f), we explore the effect of varying $d$ and observe that increasing $d$ gives increasing $\lambda_1$ which results in decreasing tail-index $\eta$.

\section{Conclusion}
We have introduced a new method, namely a comparison result in convex stochastic order for homogenized stochastic gradient descent, to obtain explicit upper bounds for the tail-index of stochastic gradient descent. These upper bounds are complemented by lower bounds obtained from results on moment stability of stochastic differential equations. Together, these bounds confirm the heavy-tailed nature of neural network parameters under optimization by SGD and provide insights into the dependency between their tail-index and optimization meta-parameters. One limitation of the method is that we have derived it only for plain SGD with constant learning rate. In future research, the method could be adapted to more advanced optimization methods involving momentum and adaptive choice of learning rate.

\begin{figure}[htbp]
\centering
\subfigure[]{
\includegraphics[width=0.45\columnwidth]{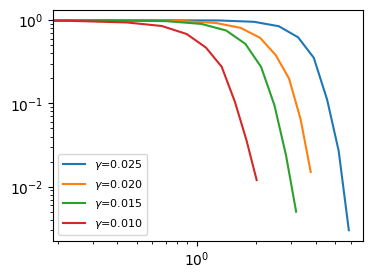}
}
\hspace{0in}
\subfigure[]{
\includegraphics[width=0.45\columnwidth]{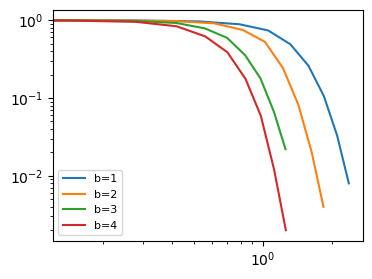}
}
\hspace{0in}
\subfigure[]{
\includegraphics[width=0.45\columnwidth]{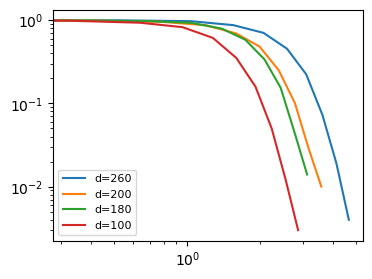}
}
\hspace{0in}
\subfigure[]{
\includegraphics[width=0.45\columnwidth]{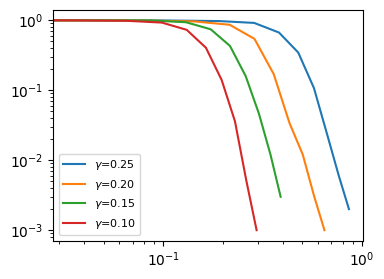}
}
\hspace{0in}
\subfigure[]{
\includegraphics[width=0.45\columnwidth]{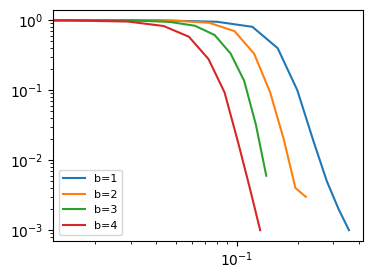}
}
\hspace{0in}
\subfigure[]{
\includegraphics[width=0.45\columnwidth]{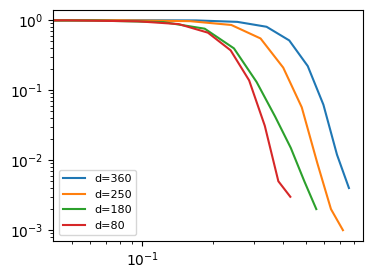}
}
\caption{Empirical complementary cumulative distribution functions on log-log scale for the effect of varying parameters.}
\label{fig_2}
\end{figure}

\bibliography{myref}
\bibliographystyle{plainnat}

%%%%%%%%%%%%%%%%%%%%%%%%%%%%%%%%%%%%%%%%%%%%%%%%%%%%%%%%%%%%%%%%%%%%%%%%%%%%%%%
%%%%%%%%%%%%%%%%%%%%%%%%%%%%%%%%%%%%%%%%%%%%%%%%%%%%%%%%%%%%%%%%%%%%%%%%%%%%%%%
% APPENDIX
%%%%%%%%%%%%%%%%%%%%%%%%%%%%%%%%%%%%%%%%%%%%%%%%%%%%%%%%%%%%%%%%%%%%%%%%%%%%%%%
%%%%%%%%%%%%%%%%%%%%%%%%%%%%%%%%%%%%%%%%%%%%%%%%%%%%%%%%%%%%%%%%%%%%%%%%%%%%%%%
\newpage
\appendix
\onecolumn

\section{Supplementary material}

\subsection{Covariance matrix}\label{sm:converance}
Consider the minibatch stochastic gradient 
\[
	\nabla \tilde{f}_{k}(x) = \frac{1}{B}\sum_{i\in\Omega_k}\nabla f_{i}(x) = \frac{1}{B}\sum_{i\in\Omega_k}\nabla L_{i}(x) +\delta x.
\]
where $B$ is the batchsize and the random set $\Omega_k = \{i_1, \cdots, i_B\}$ consists of $B$ independently identically distributed random integers sampled uniformly from $\{1, 2, \cdots, n \}$.

Let $\nabla \tilde{L}_{k}(x) = \frac{1}{B}\sum_{i\in\Omega_k}\nabla L_{i}(x)$. It can be rewritten as
\[
	\nabla \tilde{L}_{k}(x) = \frac{1}{B}\sum^{n}_{i=1}\nabla L_{i}(x)\mathrm{s}_i,
\]
where the random variable $\mathrm{s}_i = l$ if $l$-multiple $i$'s are sampled in $\Omega_k$, with $0\leqslant l \leqslant B$. The probability of $\mathrm{s}_i = l$ is given by the multinomial distribution
$\mathbb{P}(\mathrm{s}_i = l) = C_{B}^{l}(\frac{1}{n})^l (1-\frac{1}{n})^{B-l}$.
Moreover, we have
\[
	\mathbb{E}[\mathrm{s}_i] = \frac{B}{n}, \quad \mathbb{E}[\mathrm{s}_i\mathrm{s}_j] = \frac{B(B-1)}{n^2}, \quad \mathbb{E}[\mathrm{s}_i\mathrm{s}_i] =\frac{Bn + B(B-1)}{n^2}.
\]
We can also compute
\begin{equation} \label{A1}
	\mathbb{E}[\nabla \tilde{L}_{k}(x)] =  \frac{1}{B}\sum^{n}_{i=1}\nabla L_{i}(x)\mathbb{E}[\mathrm{s}_i] = \frac{1}{n}\nabla L(x)
\end{equation}
and
\begin{equation} 
\begin{aligned}\label{A2}
	&\mathbb{E}[\nabla \tilde{L}_{k}(x)^{\mathrm{T}}\nabla \tilde{L}_{k}(x)] \\
	%&= \frac{1}{B^2}\mathbb{E}\left[\sum^{n}_{i=1}\nabla L_{i}(x)^{\mathrm{T}}\mathrm{s}_i \sum^{n}_{j=1}\nabla L_{j}(x)\mathrm{s}_j\right]\\
	& = \frac{1}{B^2}\mathbb{E}\left[\sum^{n}_{i=1}\sum^{n}_{j=1}\nabla L_{i}(x)^{\mathrm{T}}\nabla L_{j}(x)\mathrm{s}_i \mathrm{s}_j\right] 
	= \frac{1}{B^2}\sum^{n}_{i=1}\sum^{n}_{j=1}\left[\nabla L_{i}(x)^{\mathrm{T}}\nabla L_{j}(x)\mathbb{E}(\mathrm{s}_i \mathrm{s}_j)\right] \\
	&= \frac{1}{B^2}\sum^{n}_{i, j=1}\nabla L_{i}(x)^{\mathrm{T}}\nabla L_{j}(x)\frac{B(B-1)}{n^2}\\
	&\quad + \frac{1}{B^2}\sum^{n}_{i=1}\nabla L_{i}(x)^{\mathrm{T}}\nabla L_{i}(x)\left[\frac{Bn + B(B-1)}{n^2} - \frac{B(B-1)}{n^2}\right] \\
	&= \frac{B-1}{B}\frac{1}{n^2}\nabla L(x)^{\mathrm{T}} \nabla L(x)  + \frac{1}{nB}\sum^{n}_{i=1}\nabla L_{i}(x)^{\mathrm{T}}\nabla L_{i}(x).
\end{aligned}
\end{equation}
Combining (\ref{A1}) with (\ref{A2}) gives
\begin{eqnarray*}
\begin{aligned}
	C(x)  &=  \mathbb{E}\left\{[\nabla \tilde{f}_{k}(x) - \nabla f(x)]^{\mathrm{T}} [ \nabla \tilde{f}_{k}(x) - \nabla f(x)] \right\}\\
		&=  \mathbb{E}\left\{ [\nabla \tilde{L}_{k}(x) - \frac{1}{n}\nabla L(x)]^{\mathrm{T}} [ \nabla \tilde{L}_{k}(x) - \frac{1}{n}\nabla L(x)] \right\} \\
		&= \mathbb{E}[\nabla \tilde{L}_{k}(x)^{\mathrm{T}}\nabla \tilde{L}_{k}(x)] - \frac{1}{n^2}\nabla L(x)^{\mathrm{T}} \nabla L(x) \\
		&= \frac{1}{B}\left[\frac{1}{n}\sum_{i=1}^{n} \nabla L_{i}(x)^{\mathrm{T}}  \nabla L_{i}(x)  -\frac{1}{n^2}\nabla L(x)^{\mathrm{T}} \nabla L(x)\right].
\end{aligned}
\end{eqnarray*}

\subsection{Proofs of some lemmas}
\label{sm:lemma_proof}
\begin{proof}[Proof of Lemma \ref{lem:POg}]
For better readability we suppress the supperscript and subscript $i$ in the following SDE
\begin{align*}
	d\hat{Z}^i_t& = b_i(\hat{Z}_t)dt + \hat{\sigma}_i(\hat{Z}^i_t)dB^i_t,
\end{align*}
where $b_i(z) = -\theta_i(z_i - \mu_i)$ and $\hat{\sigma}_i(z_i)^2 = 2\theta_i a_i (z_i^2 + 1)$.
We consider its Euler-Maruyama approximation
\[
	\hat{Z}_{K, t_{j+1}} = \hat{Z}_{K, t_{i}} + b(\hat{Z}_{K, t_{j}})\Delta t_j + \hat{\sigma}(\hat{Z}_{K, t_{j}})(B_{t_{j+1}} - B_{t_j})
\]
with $t_j =j\frac{T-t}{K} + t$, $j= \{0, 1, \cdots, K\}$ and $\Delta t_j = \frac{T-t}{K}:=\Delta$. Using Theorem 9.7.4 in \cite{KP} we have
\begin{equation}
	\mathcal{G}_{K}(t, x) = \mathbb{E}[g(\hat{Z}_{K, T})|\hat{Z}_{K, t} = x] \rightarrow \mathcal{G}(t, x), \quad t\in[0, T]. \label{POg1}
\end{equation}

Let $\mathcal{A}$ be a transition operator given by
\[
	\mathcal{A}S = S + \Delta b(S) + \hat{\sigma}(S)W
\]
with $W \sim N(0, \Delta)$. We will show that $\mathcal{A}$ satisfies the convex-ordering property
\begin{equation} \label{POg2}
	\mathbb{E}h(S_1) \leqslant \mathbb{E}h(S_2) \Rightarrow \mathbb{E}h(\mathcal{A}S_1) \leqslant\mathbb{E}h(\mathcal{A}S_2)
\end{equation}
for any convex function $h(\cdot)$. 
Let $S_1$, $S_2$ be random vectors which are independent of $W$ and satisfy $\mathbb{E}h(S_1) \leqslant \mathbb{E}h(S_2)$. Due to Stassen's theorem in \cite{Sta65}, we can also assume that $\mathbb{E}(S_2 | S_1) = S_1$. 
It follows from conditional Jensen's inequality that
\begin{equation}\label{POg5}
\begin{aligned}
	\mathbb{E}h(\mathcal{A}S_2) &= \mathbb{E}h(S_2 + \Delta b(S_2) + \hat{\sigma}(S_2)W)\\
						      &= \mathbb{E}[\mathbb{E}h(S_2 + \Delta b(S_2) + \hat{\sigma}(S_2)W) | S_1]\\
						      &\geqslant \mathbb{E}[h(\mathbb{E}(S_2 | S_1) +  \Delta \mathbb{E}(b(S_2) | S_1) +\mathbb{E}( \hat{\sigma}(S_2) | S_1)W )]\\
						      &=\mathbb{E}[h( S_1 +  \Delta b(S_1) +\mathbb{E}(\hat{ \sigma}(S_2) | S_1)W )]
\end{aligned}
\end{equation}
Here, the linearity of $b(\cdot)$ implies $\mathbb{E}(b(S_2) | S_1) = b(S_1)$. Note that the function $f(x) = \sqrt{x^2 + 1}$ is convex thanks to
\[
	%\partial_x f(x) = \frac{x}{\sqrt{x^2 + 1}},\quad
	f^{\prime\prime}(x) = \frac{2}{(x^2 + 1)\sqrt{x^2 + 1}} > 0.
\]
Similarly, $\sigma(\cdot)$ is convex. Using conditional Jensen's inequality again gives
\begin{equation} \label{POg4}
	\varpi(S_1):= \mathbb{E}(\hat{\sigma}(S_2) | S_1) \geqslant \hat{\sigma}( \mathbb{E}( S_2 | S_1) )= \hat{\sigma}(S_1). 
\end{equation}
Due to
\[
	S_1 +  \Delta b(S_1) + \mathbb{E}(\hat{\sigma}(S_2) | S_1)W \sim N(\mu, \varpi^2), \quad S_1 +  \Delta b(S_1) +\hat{\sigma}(S_1)W\sim N(\mu, \hat{\sigma}^2)
\]
with $\mu = \mathbb{E}(S_1 +  \Delta b(S_1))$,
by Theorem 3.4.7 in \cite{MS02}, (\ref{POg4}) implies that
\[ 
	\mathbb{E}[h( S_1 +  \Delta b(S_1) +\mathbb{E}( \hat{\sigma}(S_2) | S_1)W )] \geqslant \mathbb{E}[h( S_1 +  \Delta b(S_1) + \hat{\sigma}(S_1)W] = \mathbb{E}h(\mathcal{A}S_1).
\]
Combined with (\ref{POg5}) we have proved the convex-ordering property (\ref{POg2}).

By the Markov property of the Euler-Maruyama approximation we have
\[
	\mathcal{G}_{K}(t, x) = \mathbb{E}[g(\mathcal{A}^{K-1}x)].
\]
Let $\mathrm{z}$ be a Bernoulli random variable which takes the value $\mathrm{z}_1 \in \mathbb{R}$ with probability $\mathrm{p}\in (0, 1)$ and the value $z_1 \in \mathbb{R}$ with probability $1- \mathrm{p}$. Then $\mathbb{E}(Z) = \mathrm{p}z_1 +(1-\mathrm{p})z_2$. Then we have 
\[
	h(\mathbb{E}(Z)) = h(\mathrm{p}z_1 +(1-\mathrm{p})z_2) \leqslant \mathrm{p} h(z_1) + (1-\mathrm{p})h(z_2) = \mathbb{E}h(Z).
\]
Using the convex-ordering property (\ref{POg2}) of the operator $\mathcal{A}$ we obtain
\begin{equation}
	\mathcal{G}_{K}(t, \mathrm{p}z_1 +(1-\mathrm{p})z_2) = \mathcal{G}_{K}(t, \mathbb{E}(Z))= \mathbb{E}[g(\mathcal{A}^{K-1}\mathbb{E}(Z))] \leqslant  \mathbb{E}[g(\mathcal{A}^{K-1}Z)] =\mathcal{G}_{K}(t, Z) \label{POg3}
\end{equation}
due to $g$ is convex. Take expectation on both sides of (\ref{POg3}) gives
\[
	\mathcal{G}_{K}(t, \mathrm{p}z_1 +(1-\mathrm{p})z_2) \leqslant  \mathbb{E}[\mathcal{G}_{K}(t, Z)] = \mathrm{p}\mathcal{G}_{K}(t, z_1) + (1-\mathrm{p})\mathcal{G}_{K}(t, z_2), 
\]
which means $\mathcal{G}_{K}(t, \cdot)$ is convex. The approximation property (\ref{POg1}) implies the convexity of $\mathcal{G}(t, \cdot)$.
\end{proof}

\begin{proof}[Proof of Lemma \ref{lem:class_DL}]
%Let $g(z_i) = |z_i|^p$, We will show that $\mathcal{G}_i(t, Z^i_t)_{t\in[0, T]}$ is locally $L^2$-bounded. 
Since the solution to (\ref{PearsonIV}) is a polynomial process (see example 3.6 in \cite{CKT2012}), from Theorem 3.1 in \cite{FL2016} it implies
\[
	\mathcal{G}_i(t, Z^i_t) =  \mathbb{E}[g(\hat{Z}^i_T)|\hat{Z}^i_t = Z^i_t] =\exp\{(T-t)G\}\mathrm{P}(Z^i_t),
\]
where 
\begin{eqnarray*}
G=
\left(
\begin{array}{cccccc} 
0&\mathrm{g}_0&2\times 1\mathrm{g}_1&0&\cdots&0\\
0&\mathrm{g}_2&2\mathrm{g}_0 &3\times 2\mathrm{g}_1&0&\vdots\\
0&0&2\left(\mathrm{g}_2 +\mathrm{g}_3\right)&3\mathrm{g}_0&\ddots&0\\
0&0&0&3\left(\mathrm{g}_2+2\mathrm{g}_3\right)&\ddots&p(p-1)\mathrm{g}_1\\
\vdots&&&0&\ddots&p\mathrm{g}_0\\
0&&\cdots&&0&p\left(\mathrm{g}_2 +(p-1)\mathrm{g}_3\right)\\
\end{array}
\right)
\end{eqnarray*}
with
\begin{eqnarray*}
\mathrm{g}_0=\theta_i\mu_i, \quad \mathrm{g}_1= \mathrm{g}_3 = \theta_i a_i,\quad \mathrm{g}_2= -\theta_i,
\end{eqnarray*}
and $\mathrm{P}(Z^i_t)=(0, 1, Z^i_t, (Z^i_t)^2, \cdots, (Z^i_t)^p)^{\mathrm{T}}$. Then there is a constant $C_T$ that depends on $T$ such that
\[
	|\mathcal{G}_i(t, Z^i_t) |\leqslant C_T(1+|Z^i_t|^p).
\]
Let $\tau_n$ be a localizing sequence for $\mathcal{G}(t, y_t)$. Then we have
\[
	|\mathcal{G}_i(t\land\tau_n, Z^i_{t\land\tau_n}) |\leqslant C_T(1+|Z^i_{t\land\tau_n}|^p),
\]
which implies
\begin{equation} \label{comparison2prf1}
	|\mathcal{G}_i(t\land\tau_n, Z^i_{t\land\tau_n}) |^2 \leqslant C_T(1+|Z^i_{t\land\tau_n}|^{2p}).
\end{equation}
Taking $\mathcal{F}_0$-condition on both sides of (\ref{comparison2prf1}) gives
\begin{eqnarray*}
\begin{aligned}
	\mathbb{E}\left\{|\mathcal{G}_i(t\land\tau_n, Z^i_{t\land\tau_n}) |^2\right\}& \leqslant C_T\left(1+\mathbb{E}|Z^i_{t\land\tau_n}|^{2p} \right)\\
	&\leqslant C_T\left(1+\mathbb{E}\left[\sup_{n}|Z^i_{t\land\tau_n}|^{2p}\right] \right)\\
	&\leqslant C_T e^{CT}.
\end{aligned}
\end{eqnarray*}
Here, the last inequality holds based on Lemma 2.17 in \cite{CKT2012}. Thus, we complete the proof of this lemma.
\end{proof}

\subsection{Lower bound}
\label{subsec:lower_bound}
Let
\[	
	M(x) := \frac{x^{\mathrm{T}}A^{\mathrm{T}}A x }{|x|^2}, \quad x\in\mathbb{R}^{d} \setminus \{0\}
\]
denote the Rayleigh-quotient of $A^{\mathrm{T}}A$. From Chapter~1 in \citep{horn2012matrix} we have that the range of $M(x)$ is equal to the line segment $[\lambda_d^2, \lambda^2_1]$, i.e., 
\begin{equation}
\big\{M(x): x \in \mathbb{R}^d \setminus \{0\}\big\} = [\lambda_d^2, \lambda_1^2] 
\end{equation}
Evaluating the condition \eqref{ass_5_1}, we have 
\begin{eqnarray*}
\begin{aligned}
&\frac{(1 + |x|^2)\left[2x^{\mathrm{T}}F(x)\right]}{|x|^4}\\
&=\frac{(1 + |x|^2)\left\{-2\gamma x^{\mathrm{T}}\left[\frac{1}{n}A^{\mathrm{T}}(Ax - b) + \delta x\right] \right\}}{|x|^4}\\
&=\frac{(1 + |x|^2)\left[-2\gamma x^{\mathrm{T}}(\frac{1}{n}A^{\mathrm{T}}A +\delta I_d)x + 2\frac{\gamma}{n}x^{\mathrm{T}}A^{\mathrm{T}}b \right]}{|x|^4} \\
&= -\frac{2\gamma x^{\mathrm{T}}(\frac{1}{n}A^{\mathrm{T}}A +\delta I_d)x  }{|x|^4} - \frac{2\gamma x^{\mathrm{T}}(\frac{1}{n}A^{\mathrm{T}}A +\delta I_d)x}{|x|^2} + \frac{2\frac{\gamma}{n}(1 + |x|^2) x^{\mathrm{T}}A^{\mathrm{T}}b }{|x|^4}
\end{aligned}
\end{eqnarray*}
and
\begin{eqnarray*}
\begin{aligned}
&\frac{(1 + |x|^2)|G(x)|^2 - (2-\rho)|x^{\mathrm{T}}G(x)|^2}{|x|^4}\\
&=\frac{(1 + |x|^2)\left[\frac{\gamma^2}{n^2B}|\sqrt{\|Ax - b\|^2A^{\mathrm{T}}A}|^2 \right]- (2-\rho)\frac{\gamma^2}{n^2B}|x^{\mathrm{T}}\sqrt{\|Ax - b\|^2A^{\mathrm{T}}A}|^2}{|x|^4}\\
&=\frac{\frac{\gamma^2}{n^2}(1 + |x|^2)\|Ax - b\|^2|\sqrt{A^{\mathrm{T}}A}|^2 - (2-\rho)\frac{\gamma^2}{n^2}\|Ax - b\|^2|x^{\mathrm{T}}\sqrt{A^{\mathrm{T}}A}|^2}{|x|^4}\\
&=\frac{\frac{\gamma^2}{n^2B}\|Ax - b\|^2|\sqrt{A^{\mathrm{T}}A}|^2 }{|x|^4} +\frac{\frac{\gamma^2}{n^2B}\|Ax - b\|^2|\sqrt{A^{\mathrm{T}}A}|^2}{|x|^2} \\
&\qquad- \frac{(2-\rho)\frac{\gamma^2}{n^2B}\|Ax - b\|^2}{|x|^2} \frac{x^{\mathrm{T}}A^{\mathrm{T}}Ax}{|x|^2}.
\end{aligned}
\end{eqnarray*}
With $|\sqrt{A^{\mathrm{T}}A}|^2 = \mathrm{tr}(A^{\mathrm{T}}A)$ and the positive constant $\rho$ given below,
we obtain
\begin{equation}\label{5_1}
\begin{aligned}
	&\limsup_{|x |\rightarrow\infty}\frac{(1 + |x|^2)\left[2x^{\mathrm{T}}F(x)+ |G(x)|^2\right]- (2-\rho)|x^{\mathrm{T}}G(x)|^2}{|x|^4} \\
	&=\limsup_{|x |\rightarrow\infty}\Big[ - \frac{2\gamma x^{\mathrm{T}}(\frac{1}{n}A^{\mathrm{T}}A +\delta I_d)x}{|x|^2} + \frac{\frac{\gamma^2}{n^2B}\|Ax - b\|^2|\sqrt{A^{\mathrm{T}}A}|^2}{|x|^2} + \\
	&\quad + - \frac{(2-\rho)\frac{\gamma^2}{n^2B}\|Ax - b\|^2}{|x|^2} \frac{x^{\mathrm{T}}A^{\mathrm{T}}Ax}{|x|^2} \Big]\\
	&=-\frac{\gamma^2}{n^2B} \liminf_{|x| \to \infty}\left[\frac{2nB(M(x) + n\delta)}{\gamma} - \mathrm{tr}(A^{\mathrm{T}}A) M(x) + (2-\rho)M(x)^2\right] = \\
	&=-\frac{\gamma^2}{n^2B} \inf_{m \in [\lambda_d^2, \lambda_1^2]} q(m,\rho),
\end{aligned}
\end{equation}
where
\begin{equation}
q(m,\rho) = \frac{2nB(m + n\delta)}{\gamma} - \mathrm{tr}(A^{\mathrm{T}}A) m + (2-\rho)m^2.
\end{equation}
Set
\[
	\vartheta :=  2 + \frac{2nB(\lambda_1^2+ n\delta)}{\gamma \lambda_1^4} - \frac{\sum_{i=1}^{d}\lambda_i^2}{\lambda_1^2}.
\]
Note that due to the assumption $\gamma < \bar \gamma$ we have $\vartheta > 2$. We claim that 
\begin{equation}\label{eq:inf1}
\inf_{m \in [\lambda_d^2, \lambda_1^2]} q(m,\rho) > q(\lambda_1^2, \theta) = 0 
\end{equation}
for all $\rho \in [2,\vartheta)$. First, note that $m \mapsto q(m,\rho)$ is concave for any  $\rho \in [2,\vartheta)$, such that its minimum must be attained at one of the boundary values $m \in \{\lambda_d^2, \lambda_1^2\}$. Second, note that $\rho \mapsto q(m,\rho)$ is strictly decreasing for any $m \in (0, \infty)$, such that for \eqref{eq:inf1} it is sufficient to show
\begin{equation}\label{eq:inf2}
q(\lambda_d^2, \theta) \ge q(\lambda_1^2, \theta) = 0.
\end{equation} 
Using the assumption $\gamma < \bar \gamma$ we obtain
\begin{align*}
q(\lambda_d^2, \theta) &= \frac{2nB}{\gamma}(\lambda_d^2 + n\delta) - \mathrm{tr}(A^{\mathrm{T}}A)  \lambda_d^2 +  \frac{2nB}{\gamma}(\lambda_1^2 + n\delta) \frac{\lambda_d^4}{\lambda_1^4} - \mathrm{tr}(A^{\mathrm{T}}A)  \frac{\lambda_d^4}{\lambda_1^2} \ge \\
&\ge \mathrm{tr}(A^{\mathrm{T}}A) \left(\frac{(\lambda_d^2 + n\delta)}{(\lambda_1^2 + n\delta)}\lambda_1^2 - \lambda_d^2\right).
\end{align*}
For $\delta = 0$ the right hand side vanishes and \eqref{eq:inf2} is shown. Differentiation shows that the right hand side is increasing in $\delta$, such that \eqref{eq:inf2} holds for all $\delta \ge 0$. 
Altogether, we have shown that the right hand side of \eqref{5_1} is strictly negative. Thus, the SDE (\ref{eq:sde2}) satisfies the Assumption 5.1 in \cite{LMY2019}.
Based on Theorem 5.2 in \cite{LMY2019}, the solution $X_t$ of the SDE (\ref{eq:sde2}) satisfies
\begin{eqnarray*}
	\sup_{0\leqslant t < \infty}\mathbb{E}|X_t|^\rho \leqslant C
\end{eqnarray*}
for all $\rho\in [2, \vartheta)$. Therefore, the lower bound, denoted by $\eta_*$, for the asymptotic tail-index of $X_t$ is
\[
\eta_* = \vartheta = 1 + \frac{2nB(\lambda_1^2 + n\delta)}{\gamma \lambda_1^4} - \frac{ \sum_{i=2}^{d}\lambda_i^2}{\lambda_1^2}.
\]

%%%%%%%%%%%%%%%%%%%%%%%%%%%%%%%%%%%%%%%%%%%%%%%%%%%%%%%%%%%%

\end{document}